\documentclass{article} 
\usepackage{iclr2022_conference,times}

\usepackage{amsmath,amsfonts,bm}

\def\eqref#1{equation~\ref{#1}}

\def\1{\bm{1}}

\DeclareMathAlphabet{\mathsfit}{\encodingdefault}{\sfdefault}{m}{sl}
\SetMathAlphabet{\mathsfit}{bold}{\encodingdefault}{\sfdefault}{bx}{n}

\newcommand{\E}{\mathbb{E}}

\usepackage{hyperref}
\usepackage{url}
\usepackage{mathtools}
\usepackage[utf8]{inputenc} 
\usepackage[T1]{fontenc}    
\usepackage{hyperref}       
\usepackage{url}            
\usepackage{booktabs}       
\usepackage{amsfonts}       
\usepackage{nicefrac}       
\usepackage{microtype}      
\usepackage{xcolor}         
\usepackage{amsmath,amssymb}
\usepackage{graphicx}
\usepackage{amsfonts}
\usepackage{amsmath,amssymb}
\usepackage{wrapfig}
\usepackage{graphicx}
\usepackage{subcaption}
\usepackage{algorithm} 
\usepackage{algorithmic}  
\usepackage{tabularx}
\usepackage{multirow}
\usepackage{amsthm}
\usepackage{thmtools}

\newcommand{\pol}[0]{\pmb{\pi}}

\newcommand{\norm}[1]{\left\lVert#1\right\rVert}

\newtheorem{proposition}{Proposition}

\declaretheoremstyle[%
  spaceabove=-6pt,%
  spacebelow=6pt,%
  headfont=\normalfont\itshape,%
  postheadspace=1em,%
  qed=\qedsymbol%
]{mystyle} 

\title{Influence-Based Reinforcement Learning for Intrinsically-Motivated Agents}

\author{Ammar Fayad\thanks{Equal Contribution} \\Dept. of EECS, Dept. of Brain and Cognitive Sciences\\
Massachusetts Institute of Technology (MIT)\\
\texttt{afayad@mit.edu} \\
\And
Majd Ibrahim$^*$\\
Department of Electrical and Computer Engineering\\

Higher Institute for Applied Sciences and Technology (HIAST) \\
\texttt{majd.ibrahim@hiast.edu.sy} \\

}

\iclrfinalcopy 

\begin{document}

\maketitle
\begin{abstract}
Discovering successful coordinated behaviors is a central challenge in Multi-Agent Reinforcement Learning (MARL) since it requires exploring a joint action space that grows exponentially with  the  number  of  agents. In this paper, we propose a mechanism for achieving sufficient exploration and coordination in a team of agents. Specifically, agents are rewarded for contributing to a more diversified team behavior by employing proper intrinsic motivation functions. To learn meaningful coordination protocols, we structure agents’ interactions by introducing a novel  framework, where at each timestep, an agent simulates counterfactual rollouts of its policy and, through a sequence of computations, assesses the gap between other agents’ current behaviors and their targets. Actions that minimize the gap are considered highly influential and are rewarded. We evaluate our approach on a set of challenging tasks with sparse rewards and partial observability that require learning complex cooperative strategies under a proper exploration scheme, such as the StarCraft Multi-Agent Challenge. Our methods show significantly improved performances over different baselines across all tasks. 

\end{abstract}
\section {Introduction}
Deep Reinforcement Learning (DRL) has been applied to solve various challenging problems, where an agent typically learns to maximize the expected sum of extrinsic rewards gathered as a result of its actions performed in the environment \citep{sutton1998introduction}. Multi-Agent Reinforcement Learning (MARL) refers to the task of training a set of agents to maximize collective and/or individual rewards, while existing in the same environment and interacting with each other.

Recent works have shown that agents with coordinated behaviors learn remarkably faster \citep{roy2019promoting} since coordination helps the discovery of effective policies in cooperative tasks. Nevertheless, achieving coordination among agents still remains a central challenge in MARL \citep{jaques2019social}. Prominent works often resort to a popular learning paradigm called Centralized Training with Decentralized Execution (CTDE) \citep{lowe2017multi, foerster2018counterfactual}, where each agent is evaluated using a centralized critic and has access to extra information about the policies of other learning agents during training. At the time of execution, policies’ actions are restricted to local information only (i.e. their own observations). To that end, we propose a novel approach that aims at promoting coordination for cooperative tasks by augmenting CTDE MARL main return-maximization objective with an additional multi-agent objective that acts as a policy regularizer; we refer to the latter objective as the $influence ~function$.  To build intuition, a chosen agent, which we call the “influencer”, assesses the progress that other agents are making given its current policy and consequently learns behaviors that will result in an improved  performance of its teammates.
Concretely, we formulate the influence of an influencer $\pi$ as an estimation of the dissimilarity between other agents’ behaviors and their targets given the current behavior of $\pi$. The influencer is encouraged to learn behaviors that are expected to minimize that dissimilarity.  We also propose two approaches to estimate the influence and empirically show that they yield unbiased estimates of the true value. 

To that end, agents acting upon the proposed coordination paradigm learn to efficiently exploit the observed joint action space using available information. However, and since the joint space grows exponentially with the number of agents, it is highly unlikely that agents will have access to sufficient information to learn optimal behaviors to solve the task at hand; this problem arises in many scenarios such as sparse-reward environments, thus a proper exploration scheme is often required.  However, many existing multi-agent deep reinforcement learning algorithms still use mostly noise-based  techniques \citep{liu2021cooperative, rashid2018qmix, yang2018mean}. Moreover, independent exploration proved to be inefficient in cooperative settings \citep{roy2019promoting}. Recently, this challenge was addressed through Intrinsic Motivation (IM) \citep{jaques2019social, du2019liir, zhou2020learning}. Many approaches employ IM to encourage exploration of state-space \citep{han2020curiosity, burda2018exploration} or state-action space \citep{fayad2021behavior} by identifying novel configurations and rewarding an agent for visiting them. We provide an extension of these ideas into multi-agent settings and further build connection between reward shaping and coordinated behavior learning, where we choose an agent to act as an influencer (i.e. regularize its standard objective using the influence function) while other agents learn to maximize the expected sum of both extrinsic and intrinsic rewards.

To sum up, our main contributions are threefold: 1) developing an influence function to promote learning coordinated behaviors and improve team performance; 2) extending exploration via random network distillation to multi-agent settings by crafting a "novelty" function that rewards under-explored behaviors; 3) formulating a novel intrinsic incentive to promote learning diverse team behaviors to help uncover complex behaviors in a collaborative way.

We demonstrate the effectiveness of our methods on a comprehensive set of challenging tasks which include, but not limited to, the StarCraft Multi-Agent Challenges (SMAC) \citep{samvelyan2019starcraft} and the Multi-Agent Particle Environments (MAPE) \citep{mordatch2018emergence, lowe2017multi}. Empirical results show a significant improvement over a wide variety of state-of-the-art MARL approaches. We also conduct insightful ablation studies to understand the relative importance of each component of the approach individually. 

\section{Background}
\subsection{Markov games} 
Also called Stochastic games \citep{littman1994markov}, are the foundation for much of the research in multi-agent reinforcement learning. Markov games are a superset of Markov decision process (MDPs) and matrix games, including both multiple agents and multiple states. Formally, a Markov game consists of a tuple $\langle N,S,A,T,R \rangle$ where: $N$ is a finite set of agents $|N|=n \geq 2$; $S$ is a set of states, where the initial states are determined by a distribution $\rho: S \rightarrow [0,1]$; $A=\prod_{k=1}^n A_k$ is the set of joint actions; and $T: S \times A \times S \rightarrow [0,1]$ is the transition probability function.

In a Markov game, each agent is independently choosing actions and receiving rewards. Conventionally, an agent $k$ aims to maximize its own total expected return $R_k=\sum_{t=0}^T \gamma^t r_t^{(k)}$ where $\gamma$ is a discount factor and $T$ is the time horizon. 

\subsection{Multi-Agent Deep Deterministic Policy Gradient}
MADDPG ~\citep{lowe2017multi} is a multi-agent extension of the DDPG algorithm~\citep{lillicrap2015continuous}. It adapts the CTDE paradigm, where each agent $i$ possesses its own deterministic policy $\mu^{(i)}$ for action selection and critic $Q^{(i)}$ for state-action value estimation, respectively parameterized by $\theta^{(i)}$ and $\phi^{(i)}$. All parametric models are trained off-policy from previous transitions $\zeta_t \coloneqq (\mathbf{o}_t, \mathbf{a}_t, \mathbf{r}_t, \mathbf{o}_{t+1})$ uniformly sampled from a replay buffer $\mathcal{D}$. Note that $\mathbf{o}_t \coloneqq [o_t^1, ..., o_t^N]$ is the joint observation vector and $\mathbf{a}_t \coloneqq [a_t^1, ..., a_t^N]$ is the joint action vector, obtained by concatenating the individual observation vectors $o_t^{(i)}$ and action vectors $a_t^{(i)}$ of all $N$ agents. Each centralized critic is trained to estimate the expected return for a particular agent $i$ from the Q-learning loss:
\begin{equation}
\label{eq:Lcritic}
\begin{split}
\mathcal{L}^{(i)}(\mathbf{\phi}^{(i)}) 
&= \mathbb{E}_{\zeta_t \sim \mathcal{D}} 
\left[ 
    \norm{Q^{(i)} (\mathbf{o}_t, \mathbf{a}_t; \phi^{(i)})
    - y^{(i)}_t }^2\right]
\\
y^{(i)}_t &= r_t^{(i)} + \gamma Q^{(i)} (\mathbf{o}_{t+1}, \mathbf{a}_{t+1}; \bar{\phi}^{(i)})\left|_{a_{t+1}^{(j)} = \mu_{j} (o_{t+1}^{(j)}; \bar{\theta}^{(j)})\, \forall j} \right.
\end{split}
\end{equation}
Each policy is updated to maximize the expected discounted return of the corresponding agent $i$ :
\begin{equation}
    J_{PG}^{(i)}(\mathbf{\theta}^{(i)}) = \mathbb{E}_{\mathbf{o}_t \sim \mathcal{D}} \left[Q^{(i)} (\mathbf{o}_t, \mathbf{a}_t) \bigg|_{{a_t^{(j)} = \mu_j(o_t^{(j)};\,{\theta}^{(j)})\,}} \right]
\label{eq:JPG}\end{equation}
Notice that while optimizing an agent's policy, all agents' observation-action pairs are taken into consideration. By that, the value functions of all agents are trained in a centralized, stationary environment, despite happening in a multi-agent setting. Moreover, this procedure allows for the learning of coordinated strategies, yet needs to be augmented with efficient exploration methods that reward novel action configurations which may lead to the discovery of higher-return behaviors.

\section{Methods}
\subsection{Basic Influence}
Intuitively, one can define coordination in a team of agents as the behavior of each individual agent being informed by other agents. Furthermore, agents' behaviors can be inter-affected either directly through communication for example or indirectly through task-specific shared goals and/or rewards  or the dynamics of the environment. We hypothesize that when agents learn in a cooperative setting, they tend to affect each other's exploitation processes, we confirm the hypothesis throughout the paper and build on that to formalize a general method to foster influential interactions and learn meaningful coordination protocols. Specifically, we introduce a novel framework to assess the influence that agent $\pi$, at timestep $t$, has on a set of agents upon taking an action $a_t^{(\pi)}$ in a global state $s_t$. More concretely, consider $n$ agents, namely $\pi, \mu_1, \mu_2, ..., \mu_{n-1}$. Define $\pmb{\mu}=[\pi,\mu_1,...,\mu_{n-1}]^T$ as the joint policy. We use this notation throughout the rest of the paper.
Essentially, the agent $\pi$, which we call the "influencer", asks a retrospective question: "How much are agents $\{\mu_k\}_{k=1}^{n-1}$ (i.e. the "influencees") expected to get closer to their target returns after $\pi$   executes an action $a_t^{(\pi)}$ in a global state $s_t$?" \footnote{In some cases, the agents might not have full access to the state $s$ information even during training. However, a straightforward approach is to substitute $s$ with a concatenation of all agents' observations.}
Meaning that state-action pairs that lead agents $\{\mu_k\}_{k=1}^{n-1}$ closer to their target returns are considered highly influential and are rewarded. The goal of this section is to show how $\pi$ can learn effective policies that drive teammates' behaviors towards their targets by estimating its influence. 
\subsubsection{Influence with Single Estimator}
Formally, we quantify the influence $F_\pi$ of agent $\pi$ on  $\{\mu_i\}_{i=1}^{n-1}$ by initializing a network $Q^{cen} : S \times A \rightarrow \mathbb{R}^{n-1}$ with parameters $\phi^{cen}$; $Q^{cen}(.;\phi^{cen})^{(i)}$ estimates the updated $q$-value of agent ${\mu_i}$ after frequent visits of $\pi$ to $(s_t,a_t^{(\pi)})$, by minimizing the following loss:
\begin{equation}
\mathcal{L}(\phi^{cen}) = \E_{(\mathbf{x},\mathbf{a},r, \mathbf{x}')\sim\mathcal{D}} \bigl[\norm{Q^{cen}(\mathbf{x},\mathbf{a};\phi^{cen})-\mathbf{y}}^2\bigl]
\end{equation}
Where $\mathbf{y}\in \mathbb{R}^{n-1}$, $\mathbf{y}^{(i)}=r^{(i)}+\gamma Q^{(i)}_{\text{target}}(\mathbf{x}',\pmb{\mu}(\mathbf{x}');\bar{\phi}^{cen})$; $Q^{(i)}_{\text{target}}$ is the target critic of agent $\mu_i$,  and $\mathcal{D}$ is a buffer containing all agents' experiences with the exception that the agent $\pi$'s experience is restricted to $(o_{\pi}(s_t),a_t^{(\pi)})$; in other words, for all $(\mathbf{x},\mathbf{a},.)\in \mathcal{D}$, $\mathbf{x}^{(\pi)}=o_{\pi}(s_t)$ and $\mathbf{a}^{(\pi)}=a_t^{(\pi)}$.

After obtaining an estimate of what the $q$-values would be after counterfactual rollouts of $\pi$ starting from $(o_\pi(s_t),a_t^{(\pi)})$, we can now compute the influence $F_\pi$:
\begin{equation}\begin{split}
&  F_{\pi}=\E_{(\mathbf{x}, . , r , \mathbf{x'}) \sim \mathcal{B}} \bigl[ \norm{Q^{cen}(\mathbf{x},\pmb{\mu}(\mathbf{x}))- \mathbf{y}}^2\bigl]\\ 
&\mathbf{y}^{(i)}=r^{(i)}+\gamma Q^{(i)}_{\text{target}}(\mathbf{x'},\pmb{\mu}(\mathbf{x}'))\end{split}
\label{inf}\end{equation} 

Where $\mathcal{B}$ is a buffer storing all agents' transitions. Minimizing $F_\pi$ as a regularizer of its return objective $J_\pi$, $\pi$ adjusts its actions' selections so that $\{\mu_k\}_{k=1}^{n-1}$ can reach their goals faster and more efficiently, thus achieving sufficient coordination. Note that the second term in the expectation (i.e. the target vector $\mathbf{y}$) is set to be undifferentiable with respect to $\pi$'s parameters and thus does not propagate through its network.
\subsubsection{Influence with Multiple Individual Estimators}
As seen earlier,  agent $\pi$ estimates the gap between each agent's value and its target value by employing a single network. Another desirable approach is to use multiple estimators where each estimator, namely $Q_{\text{clone}}^{(i)}$, individually calculates a fairly good approximation of what the $q$-value of $\mu_i$ would be after counterfactual rollouts of $\pi$ starting from $(o_\pi(s_t),a_t^{(\pi)})$. To reduce computational costs and arrive at better estimates, each estimator's network is initialized with the parameters of the corresponding critic network at each episode (i.e. $Q_{\text{clone}}^{(i)}\leftarrow Q^{(i)}$).
The training is carried out similarly to that of the single estimator setting,\begin{equation}
\mathcal{L}(\phi^{(i)}) = \E_{(\mathbf{x},\mathbf{a},r, \mathbf{x}')\sim\mathcal{D}} \bigl[||Q_{\text{clone}}^{(i)}(\mathbf{x},\mathbf{a};\phi^{(i)})-y'_i||^2\bigl]
\end{equation} 

The influence function could be expressed as:\begin{equation}
    \begin{split}
&  F_\pi=\E_{(\mathbf{x}, . , . ,\mathbf{x}') \sim \mathcal{B}}  \frac {1}{n-1} \sum_{i=1}^{n-1}  \norm{Q^{(i)}_
{\text{clone}}(\mathbf{x},\pmb{\mu}(\mathbf{x}))- y_i}^2 \\ 
& y_i= r^{(i)}+\gamma Q_{\text{target}}^{(i)}(\mathbf{x'},\pmb{\mu}(\mathbf{x}'))
\end{split}\end{equation}

\begin{wrapfigure}{R}{0.38\textwidth}
\centering
\includegraphics[width=1\linewidth]{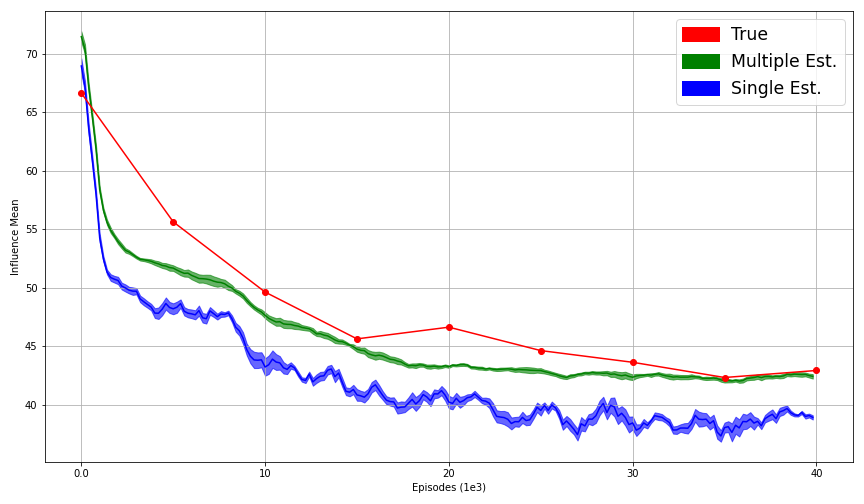}
	\caption[width=\columnwidth]{
		Empirical evaluation of the bias in the proposed methods of measuring influence. The influence values estimated by single and multiple estimators are compared to the true values of $F$. The results for
the estimated values are averaged
across 8 runs.

	}\label{figinf}
\end{wrapfigure}

\textbf{Which approach yields better estimates of the true value of the influence?}
The influence of an agent $\pi$ on a team of agents $T$ was defined as a measure of the improvement in the performance of $T$ given the current behavior of $\pi$. However, measuring the influence using function approximators might result in inaccurate estimates. To resolve this concern,  we plot the influence estimates of the two prior  approaches over time while they learn on the Cooperative Navigation MAPE task \citep{mordatch2018emergence}, where the number of agents is $N=6$. In Figure (\ref{figinf}), we graph the average influence estimates over 40000 episodes and compare it to the true value. The latter can be expressed as the average distance between the true $q$-value and the true $q$-target of each agent given the current behavior of one pre-labeled agent (i.e. the influencer). This distance is then averaged over 1000 episodes following the current policies of agents and is reported every 5000 episodes. The plots show a relatively small bias of both methods during learning. However, as Figure (\ref{figinf}) suggests, measuring influence using multiple individual estimators yielded more accurate values after enough training which substantiates its superiority over the shared network approach. 
Note that, although confirms our prior hypothesis,  this experiment does not reflect the importance of employing the influence on the final performance of the agents as we will discuss that in Section (\ref{exp}).

\subsection{Intrinsic Motivation for Diversified Team Behavior}

In this section, we introduce a framework for achieving cooperative exploration by ensuring that agents are consistently tilted towards visiting under-explored state-action configurations;  we start by providing a simple demonstration which shows that the number of environment steps required for all agents to randomly traverse all possible action configurations increases at least exponentially with the number of agents.
\begin{proposition}\label{prop1}  Consider an $L$-action setting of $n$ agents. In expectation, the number of steps $T$ needed to visit all $L^n$ action configurations at least once without coordinated exploration grows at least exponentially with the number of agents. More concretely,  $\mathbb{E}[T]=\Omega(nL^n )$.\end{proposition}\begin{proof}See Appendix \ref{APP}.
\end{proof}
To mitigate this issue, we assign agent $\pi$ a prediction error as an intrinsic reward to facilitate recognizing and learning novel behaviors:
\begin{equation}\label{distill}
    r^{(\pi)}_t=r^{\text{ext}}_t+\lambda_\pi \underbrace{\norm{\phi(o_\pi(s_t),\mathbf{a}_t^{(\pi)})-(o_\pi(s_t),\mathbf{a}_t^{(\pi)})}^2}_{\psi(s_t,\mathbf{a}_t)}
\end{equation}
Where $\phi$ is an autoencoder network regularly trained on data generated by the policy $\pi$ and $\lambda_\pi$ is a hyper-parameter that balances the extrinsic and intrinsic reward terms. $\psi$'s expression stems from the  observation that when an autoencoder is trained on data from a particular distribution, it will be good at reconstructing data from that distribution, while it will perform poorly if the data is from a different distribution \citep{fayad2021behavior,zhang2019learning}. Thus, by employing $\psi$ as an intrinsic bonus, $\pi$ rewards states' observations and actions that do not belong to the data generated by it.
In practice, $\phi$ is designed to be a relatively large network since we want it to be slightly overfitted to the training data so that it will not accidentally generalize to behaviors that we may deem novel.

Nevertheless, assigning each agent a $\psi$ is not sufficient as it makes the case equivalent to independent exploration approaches. Thus, we propose a coordinated exploration method that takes into account other agents' behaviors, encouraging agents to diversify team behavior while maintaining good performance.

Specifically, we assign agents $\{\mu_i\}_{i=1}^{n-1}$  intrinsic penalty defined as:
\begin{equation}\label{intrinsic}
r_{\mu_i}^{\text{int}}(s_t,\mathbf{a}_t)= -\exp\big(-\omega_{i}\psi(s_t,\mathbf{a}_t^{(\pi)})\big) \norm{\mu_i (o _ {\mu_i} (s_t))-\mathbf{a}_t^{(\mu_{i})}}^2 
\end{equation}
This reward term aims at teaching the agents to recognize previous behaviors and synchronously select novel configurations. To build intuition, consider a case where $N=2$. Whenever $(\pi,\mu)$ select an action tuple in the neighborhood of a frequently-visited tuple $(a_1,a_2)$ in a global state $s$,  $\psi$ will be relatively small and the penalty, $r^{\text{int}}_{\mu}$, will be large. Conversely, if $(\pi,\mu)$ encounter a novel tuple, say $(a_1',a_2')$ in $s$, the small penalty of $\mu$ (Eq. (\ref{intrinsic})) together with a large reward for $\pi$ (Eq. (\ref{distill})) will drive both agents to further explore  this encounter.

In all, Fig. (\ref{figinf}) shows how this framework can be augmented with the basic influence introduced earlier to reinforce learning and discovering coordinated behaviors.
\section{Empirical Evaluation \& Analysis}\label{exp}
The goals of our experiments are to: a) verify the performance of our method on a comprehensive set of multi-agent challenges (SMAC, MAPE, sparse-reward settings, and continuous control environments); b) perform ablations to examine which particular components of the proposed framework are important for good performance.

\begin{figure}[t]
\centering
\begin{subfigure}[b]{0.55\textwidth}
\includegraphics[width=1\linewidth]{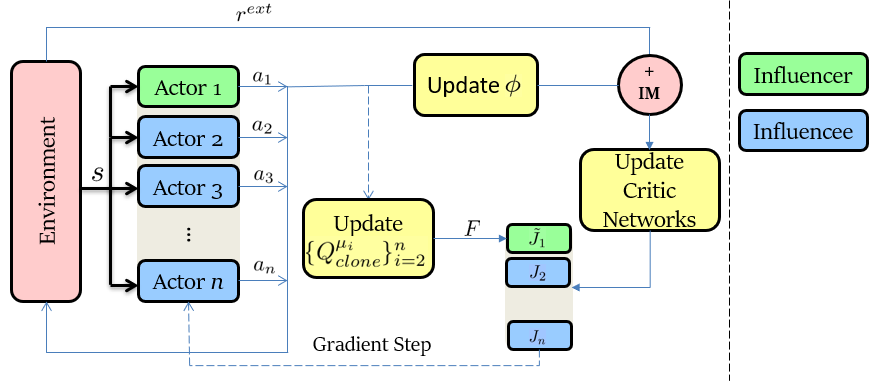}
\end{subfigure}
	\caption{
		Architecture of the general proposed method 
	}\label{figinf}

\vskip -0.1in 
\end{figure}

\subsection{Cooperative \& Mixed Games}

\subsubsection{StarCraft Multi-Agent Challenge}
StarCraft provides a rich set of heterogeneous units each with diverse actions, allowing for extremely complex cooperative behaviors among agents. We thus evaluate our method on several SC micromanagement tasks from the SMAC  \footnote{\url{https://github.com/oxwhirl/smac}} benchmark  \citep{samvelyan2019starcraft}, where a group of mixed-typed units controlled by decentralized agents needs to cooperate to defeat another group of mixed-typed enemy units controlled by built-in heuristic rules with “difficult” setting; the battles can be both symmetric (same units in both groups) or asymmetric. Each agent observes its own status and, within its field of view, it also observes other units’ statistics such as health, location, and unit type (partial observability); agents can only attack enemies within their shooting range. A shared reward is received on battle victory as well as damaging or killing enemy units. Each battle has step limits set by SMAC and may end early. We consider 4 battle maps grouped into \textbf{Easy} (2s3z), \textbf{Hard} (5m\_vs\_6m, 3s\_vs\_5z), and \textbf{Super Hard} (corridor) against 6 baseline methods using their open-source implementations based on PyMARL  \citep{samvelyan2019starcraft}: COMA \citep{foerster2018counterfactual}, IQL \citep{Tan93multi-agentreinforcement}, VDN \citep{Sunehag2018ValueDecompositionNF}, QMIX \citep{rashid2018qmix}, LIIR (Individual Intrinsic Rewards)  \citep{du2019liir}, and LICA (Implicit Credit Assignment) \citep{zhou2020learning}.

The corridor map, in which 6 Zealots face 24 enemy Zerglings, requires agents to make effective use of the terrain features and block enemy attacks from different directions. A properly coordinated exploration scheme applied to this map would help the agents discover a suitable unit positioning quickly and improve performance, while 2s3z requires agents to learn “focus fire" and interception.  For the asymmetric 5m\_vs\_6m, basic agent coordination alone such as “focus firing” no longer suffices \citep{du2019liir} and consistent success requires extended exploration to uncover complex cooperative strategies such as pulling back units with low health during combat. The 3s\_vs\_5z scenario features three allied Stalkers against five enemy Zealots. Since Zealots counter Stalkers, the only winning strategy for the allied units is to kite the enemy around the map and kill them one after another, causing the failure of independent learning algorithms to learn good policies in this task.

For all these scenarios, our method consistently shows the best performance with significant learning speed. Detailed results are reported in Figure (\ref{smac}) as they present the median win rate of the methods during the training across 12 random runs. 

\begin{figure}[h!]
 
\centering
\begin{subfigure}[b]{0.24\linewidth}
    \includegraphics[width=\linewidth]{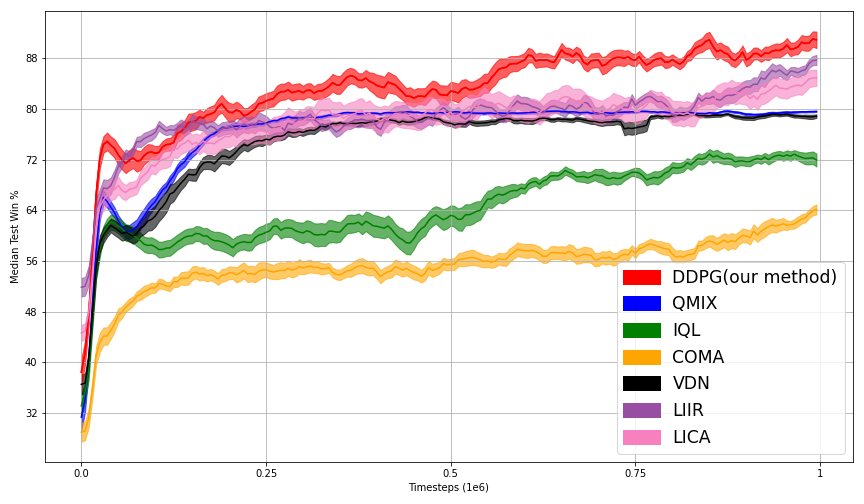}
     \caption{2s3z \textbf{Easy}}
  \end{subfigure}   
\begin{subfigure}[b]{0.24\linewidth}
    \includegraphics[width=\linewidth]{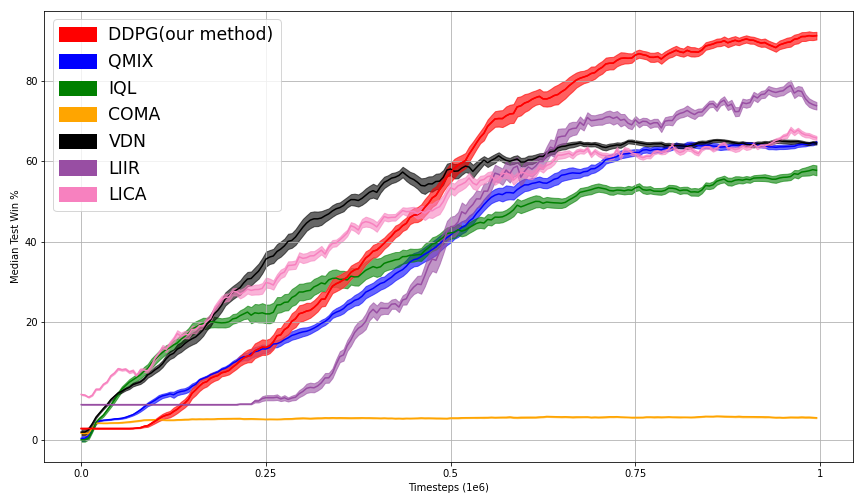}
     \caption{3s\_vs\_5z \textbf{Hard}}
  \end{subfigure}
\begin{subfigure}[b]{0.24\linewidth}
    \includegraphics[width=\linewidth]{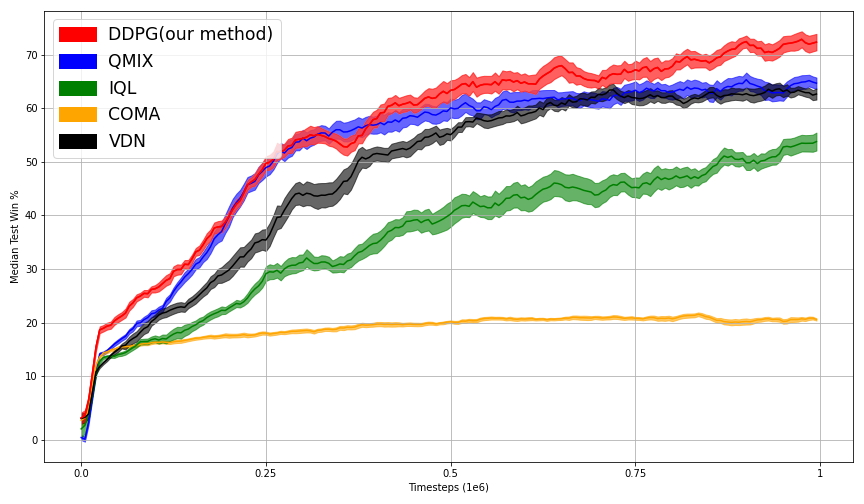}
     \caption{5m\_vs\_6m \textbf{Hard}}
\end{subfigure} 
\begin{subfigure}[b]{0.24\linewidth}
    \includegraphics[width=\linewidth]{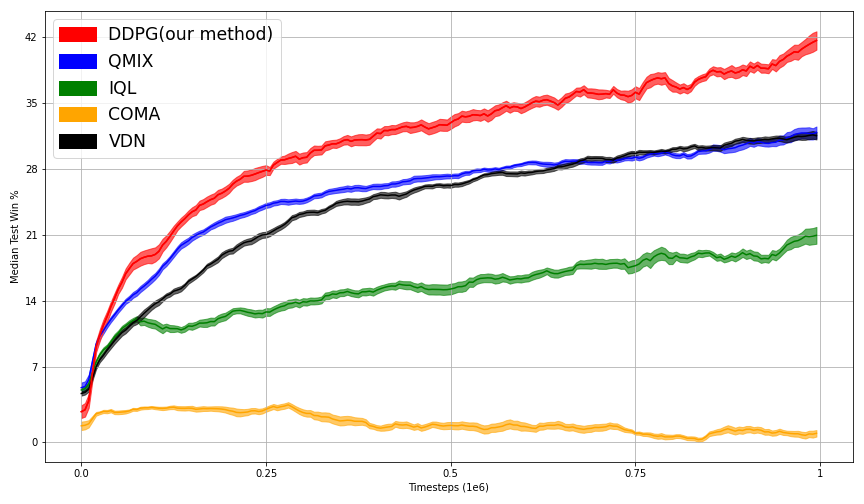}
     \caption{corridor \textbf{SuperHard}}
  \end{subfigure}
\caption{The median test win \%  of various methods across the SMAC scenarios.}
\label{smac}
\end{figure}

\subsubsection{Sparse-Reward Settings}
We test on two additional tasks to show the effectiveness our method on sparse-reward settings and compare it to famous influence-based coordinated exploration algorithms (Table (\ref{pushbox})): EDTI, EITI \citep{wang2019influence}, and Social Influence \citep{jaques2019social}.

\textbf{Sparse Push-Box:}  A $15 \times 15$ room is populated with 2 agents and 1 box.  Agents need to push the box to the wall in 300 environment steps to get a reward of 1000. Moreover, both can observe the coordinates of their teammate and the location of the box. However, the box is so heavy that only when two agents push it in the same direction at the same time can it be moved a grid. Agents need to coordinate their positions and actions for multiple steps to earn a reward.

\textbf{Sparse Secret Room:} A $25 \times 25$ grid is divided into three small rooms on the right and one large room on the left where 2 agents are initially spawned. There is one door between each  small  room  and  the  large  room.   A  switch  in  the large  room  controls  all  three  doors.   A  switch also exists in  each small room which only controls the room’s door. The agents need to navigate to one of the three small rooms, i.e. the target room, to receive positive reward. The task is considered solved if both agents are in the target room. The state vector contains $(x,y)$ locations of all agents and binary variables to indicate if doors are open.

\begin{table*}[ht!]
\centering
\caption{\label{pushbox} Results on the \textbf{Push-Box} and \textbf{Secret Room} tasks after 20000 updates across 10 runs. }

\begin{center}
\begin{small}
\begin{tabular}{lcccccc}
\toprule
& \multicolumn{2}{c}{ Push-Box} & \multicolumn{2}{c}{ Secret Room} \\
Agents & \textbf{Team Performance } & \textbf{Performance std} & \textbf{Avg Success Rate} & \textbf{Success Rate std}   \\ \hline

DDPG(our method) &  \textbf{146.66  }  & 34.13 &  \textbf{0.68  }  & 0.04\\
EDTI & \textbf{135.84   } & 45.20 & {0.34 } & 0.02\\
Social Influence &  86.67   & 65.81 &  0.25   & 0.10\\
EITI &  75.09   & 78.54 &  0.46   & 0.06\\

\bottomrule
\end{tabular}
\end{small}
\end{center}
\vskip -0.2in
\end{table*}

\subsubsection{Multi-Agent Particle Environments}\label{discrete}
To understand how the proposed method helps agents achieve cooperative behavior in nonstationary settings, we conduct experiments on the grounded communication environment  \footnote{\url{https://github.com/openai/multiagent-particle-envs}} proposed in \citep{mordatch2018emergence, lowe2017multi}. Each task consists of multiple agents ($N \geq 2$) and $L$ landmarks in a two-dimensional world with continuous space and discrete time. Both agent and landmark entities inhabit a physical location in space and posses descriptive physical characteristics, such as color and shape type. For that purpose, we adapt the DDPG algorithm as a learning framework and train with 10 random seeds. Results of the following tasks are reported in Tables (\ref{tab:simple_spread}, \ref{tab:simple_speaker_listener}, \ref{tab:simple_adversary}).

\textbf{Cooperative Navigation:} 
In this environment, $N$ agents must collaborate to reach a set of $N$ landmarks with known positions. Agents are rewarded based on how far any agent is from each landmark, meaning that the agents learn to spread with each agent covering one landmark. The agents, which occupy a significant physical space,  are aware of their relative positions to each other and are further penalized when colliding with each other.

\begin{table*}[ht!]

\caption{\label{tab:simple_spread} Avg \# of collisions per episode and avg agent distance from a landmark in the \textbf{cooperative navigation} task, after 25000 episodes.}
\centering\begin{small}
\begin{tabular}{l c c c c}
\toprule
& \multicolumn{2}{c}{ $N = 3$} & \multicolumn{2}{c}{ $N = 6$} \\
Agent $\pol$& \textbf{Average dist.}&    \textbf{\# collisions} & \textbf{Average dist.}&    \textbf{\# collisions}   \\ \hline
DDPG(our method) &  \textbf{1.559} & \textbf{0.185} & {3.349} & \textbf{1.294}  \\
MADDPG & 1.767 &0.209 & \textbf{3.345} & 1.366  \\
DDPG(No influence function) &  1.858 & 0.375 & 3.350 & 1.585 \\
\bottomrule
\end{tabular}
\end{small}
\end{table*}

\textbf{Cooperative Communication:}
Here, a stationary speaker must guide a listener in an environment consisting of three landmarks of differing colors. At each episode, one landmark of a particular color is set as a goal for the listener to be reached, however, only the speaker can observe which landmark the listener must navigate to. Moreover, The speaker can produce a communication output at each time step which is observed by the listener. The latter must navigate the environment to reach the correct landmark. Agents are collectively rewarded at the end of an episode based on the listener's distance from the correct landmark.

\begin{table*}[ht!]
\centering
\caption{\label{tab:simple_speaker_listener} Percentage of episodes where the agent reached the target landmark and average distance from the target in the \textbf{cooperative communication} environment, after 25000 episodes.}

\begin{center}
\begin{small}
\begin{tabular}{lcccccc}
\toprule
Agent & \textbf{Target reach \% }&    \textbf{Average distance}   \\ \hline
DDPG(our method) &  \textbf{90.3\%} & \textbf{0.093}  \\
MADDPG & 84.0\% & 0.133  \\
DDPG &  32.0\% & 0.456 \\
DQN & 24.8\% & 0.754 \\
Actor-Critic & 17.2\% & 2.071 \\
TRPO & 20.6\% & 1.573 \\
REINFORCE &  13.6\%  & 3.333 \\

\bottomrule
\end{tabular}
\end{small}
\end{center}
\vskip -0.1in
\end{table*}

\textbf{Physical Deception:} 
This environment consists of $N$ agents and $N$ landmarks, with one landmark as the target of all agents. The agents are rewarded based on the distance of the closest agent to the target landmark, making it sufficient for only one agent to reach it. An adversary agent also tries to reach the target landmark, while the agents are penalized as it gets closer to the target. The adversary, however, does not know which landmark is the target and must deduce it from the agents' behavior. For that reason, agents must cooperate to trick the adversary by learning to cover all the landmarks. This task shows that our algorithm  is applicable not only to cooperative interactions but to mixed environments as well.

\begin{table*}[ht!]
\centering\begin{small}

\caption{\label{tab:simple_adversary} Results on the \textbf{physical deception} task, with $N=2$ cooperative agents/landmarks. Success (\textit{succ \%}) for agents (AG) and adversaries (ADV) is if they are within a small distance from the target landmark. }
\begin{tabular}{l l c c c c c c}
\toprule

Agent $\pol$ & Adversary $\pol$ & \textbf{AG succ \% }&    \textbf{ADV succ \%}& \textbf{$\Delta$ succ \%} \\ \hline
DDPG(our method) & MADDPG & 95.2\% & 45.1\% & 50.1\%  \\
MADDPG & MADDPG & 94.4\% & 39.2\% & 55.2\%  \\
MADDPG & DDPG &  92.2\% & 16.4\% & {75.8\%} \\
DDPG & MADDPG & 68.9\%  & 59.0\% & 9.9\%  \\
DDPG & DDPG & 74.7\%  & 38.6\% & 36.1\%  \\
\bottomrule
\end{tabular}
\end{small}
\vskip -0.1in
\end{table*}

\subsection{Continuous Environments}
To confirm the scalability of our algorithm to large continuous settings, we measure the performance of our algorithm on a suite of PyBullet \citep{tan2018sim} continuous control tasks, interfaced through OpenAI Gym \citep{brockman2016openai}. Gym environments, however, are mainly single-agent settings, thus to evaluate our approach, we reframe the problem by introducing an additional learning agent that acts as an auxiliary agent. Crucially, both agents work collaboratively in order to find a region of the solution space where an agent accumulates higher rewards. We use TD3 \citep{fujimoto2018addressing} as our learning model and test it against state-of-the-art algorithms in 5 gym environments. Our algorithm outperforms all baselines across all different environments (e.g. our method attains 131\% return of SAC final performance on Humanoid-v3). For detailed results, see Appendix \ref{ADD}.


\subsection{Ablations}\label{abl}
\begin{wrapfigure}[14]{R}{0.4\textwidth}
\centering
\includegraphics[width=1\linewidth]{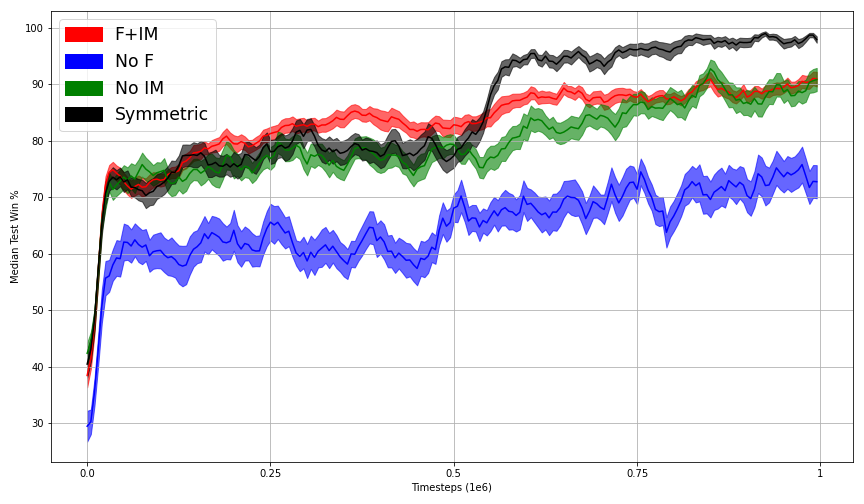}
	\caption[width=\columnwidth]{
		Ablations for different components of our framework on 2s3z scenario.
	}\label{ablfig}
\end{wrapfigure} 
We further investigate the significance of each component along with a symmetric extension of the proposed framework. Specifically, we consider the three cases: 1) \textbf{No F}: where the influence function does not contribute to the update rule to any of the policies; 2) \textbf{No IM}: where a randomly selected agent maximizes both the expected sum of $extrinsic$ rewards along with the influence function, and other agents' policies are learned using the DDPG; 3) \textbf{Symmetric}: where all agents simultaneously play the role of an influencer and influencee: they  learn to maximize an augmented reward function ($extrinsic$ and $intrinsic$) along with the influence function.

Results of the experiments conducted on the 2s3z SMAC scenario show that, in the absence of the intrinsic rewards (\textbf{No IM}), the agents experience a slightly decreased overall performance  when compared to the significant decline induced by detaching the influence function (\textbf{No F}). 

In Figure (\ref{ablfig}),  we observe that the agents following the \textbf{Symmetric} approach learn faster, and
achieve a significantly higher median win rate. This approach, however, doubles the computational costs which restricts its applicability in larger settings.
\section{Related Work}\label{related}
We discuss recently developed methods for exploration in RL using intrinsic motivation, coordination in multi-agent  RL, and influence-based coordinated exploration methods subsequently.

Intrinsic motivation (IM) has been increasingly used both in single-agent RL and multi-agent RL. A core idea of IM is to encourage the agent to take new actions or visit new states, thus exploring the environment and obtaining more diverse behaviors. One common approach is to approximate state or state-action visitation frequency and add a reward bonus to states the agent rarely covers \citep{Tang2017ExplorationAS,Bellemare2016UnifyingCE,martin2017count}. A more related IM approach is to evaluate state visitation novelty \citep{klissarov2019variational, han2020curiosity, burda2018exploration} or state-action visitation novelty \citep{fayad2021behavior}. Inspired by the latter, we provided a natural extension for this approach to the MARL settings by the learning of a "novelty" function. Other works make use of single-agent IM to construct their multi-agent intrinsic reward \citep{du2019liir,iqbal2019coordinated}. Each agent in \citep{du2019liir} learns a distinct intrinsic reward so that the agents are stimulated differently, even when the environment only feedbacks a team reward. This reward helps distinguish the contributions of the agents when the environment only returns a collective reward. In \citep{iqbal2019coordinated}, each agent has a novelty function that assesses how novel an observation is to it, based on its past experience. Their multi-agent intrinsic reward is defined based on how novel all agents consider an agent's observation. A recent work \citep{liu2021cooperative} assigns agents a common goal while exploring.  The goal is selected from multiple projected state spaces via a normalized entropy-based technique. Then, agents  are  trained  to  reach  this  goal  in  a  coordinated  manner.

Many works studied the cooperative settings in MARL; a straightforward approach is to use independent learning agents (fully decentralized learning). This approach, however, is shown to perform inadequately both with $Q$-learning \citep{matignon2012independent} and with policy gradient \citep{lowe2017multi}. Therefore, we considered the CTDE paradigm, where each agent's policy takes its individual observation as many real life applications dictate, while the centralized critic permit for sharing of information during training. Policy gradient methods have been commonly used along with the CTDE paradigm in MARL, either by implementing a single centralized critic for all agents \citep{foerster2018counterfactual}, or one centralized critic for each agent \citep{lowe2017multi}. Adopting the latter, we enable agents with different reward functions to learn in competitive and mixed scenarios as well.

Some other works encouraged cooperative interactions between agents by sharing useful information \citep{yang2020learning,hostallero2020inducing}. In \cite{hostallero2020inducing}, each agent broadcasts a signal that represents an assessment of the effect of the joint actions that all agents take on its expected reward. Different from our approach, this signal encourages agents to behave as is expected of them and does not benefit exploration.
As for \cite{yang2020learning}, each agent learns an incentive function that rewards other agents based on their actions. Each agent's function aims to alter other agents’ behavior to maximize its extrinsic rewards. To accomplish this, each agent requires access to every other agent's policy, incentive function, and return making this approach difficult to scale and execute. Additionally, \cite{roy2019promoting} proposed two policy regularizers approaches to promote coordination in a team of agent, one of which assumes that an agent must be able to predict the behavior of its teammates in order to coordinate with them, while the other supposes that coordinated agents collectively recognize different situations and synchronously switch to different sub-policies to react to them.

Similarly to our work, \citep{jaques2019social} proposed a similar idea of rewarding an agent for having a casual influence on other agents' actions. Their method showed interesting results in terms of learning coordinated behavior. However, this casual influence is designed to reward policies for influencing other policies' actions without considering the "quality" of this influence. \cite{barton2018coordination} propose causal influence as a way to measure coordination between agents, specifically using Convergence Cross Mapping (CCM) to analyze the degree of dependence between two agents’ policies. Our method also draws inspiration from the work of \citep{wang2019influence}, as they define an influence-based intrinsic exploration bonus by the expected difference between the action-value function of one agent and its counterfactual action-value function without considering the state and action of the other agent.

\section{Conclusions \& Future Work}
We introduced a novel multi-agent RL algorithm for achieving coordination through assessing the influence an agent has on other agents' behaviors. Additionally, we proposed to learn an intrinsic reward for each agent to promote coordinated team exploration. We tested our algorithm on a wide variety of tasks with many challenges, such as partial observability, sparse rewards, and large spaces; these tasks include, but not limited to, SMAC, MAPE, as well as OpenAI gym continuous environments. Our methods achieved noticeable improvement over prominent algorithms on all tasks. One promising extension of our algorithm is to use Graph Attention Networks \citep{velivckovic2017graph,Zhou2020GraphNN} to learn the importance of the influencer in determining the influencees' policies and to establish a message-passing architecture in networked systems. The investigation of the effectiveness of these methods is left for future works.

\section{Reproducibility Statement}
We have provided an illustration of the proposed algorithm in Fig. (\ref{figinf}) along with implementation details and hyperparameters selection in Appendix \ref{app:training_details}. Furthermore, code is submitted with Supplementary Material and each algorithm is evaluated at least 10 times using random seeds on all environments.
\bibliography{iclr2022_conference}

\begin{thebibliography}{46}
\providecommand{\natexlab}[1]{#1}
\providecommand{\url}[1]{\texttt{#1}}
\expandafter\ifx\csname urlstyle\endcsname\relax
  \providecommand{\doi}[1]{doi: #1}\else
  \providecommand{\doi}{doi: \begingroup \urlstyle{rm}\Url}\fi

\bibitem[Ba et~al.(2016)Ba, Kiros, and Hinton]{ba2016layer}
Jimmy~Lei Ba, Jamie~Ryan Kiros, and Geoffrey~E Hinton.
\newblock Layer normalization.
\newblock \emph{arXiv preprint arXiv:1607.06450}, 2016.

\bibitem[Barton et~al.(2018)Barton, Waytowich, and
  Asher]{barton2018coordination}
Sean~L Barton, Nicholas~R Waytowich, and Derrik~E Asher.
\newblock Coordination-driven learning in multi-agent problem spaces.
\newblock \emph{arXiv preprint arXiv:1809.04918}, 2018.

\bibitem[Bellemare et~al.(2016)Bellemare, Srinivasan, Ostrovski, Schaul,
  Saxton, and Munos]{Bellemare2016UnifyingCE}
Marc~G. Bellemare, Sriram Srinivasan, Georg Ostrovski, Tom Schaul, David
  Saxton, and R{\'e}mi Munos.
\newblock Unifying count-based exploration and intrinsic motivation.
\newblock In \emph{30th Conference on Neural Information Processing Systems},
  2016.

\bibitem[Brockman et~al.(2016)Brockman, Cheung, Pettersson, Schneider,
  Schulman, Tang, and Zaremba]{brockman2016openai}
Greg Brockman, Vicki Cheung, Ludwig Pettersson, Jonas Schneider, John Schulman,
  Jie Tang, and Wojciech Zaremba.
\newblock Openai gym.
\newblock \emph{arXiv preprint arXiv:1606.01540}, 2016.

\bibitem[Burda et~al.(2019)Burda, Edwards, Storkey, and
  Klimov]{burda2018exploration}
Yuri Burda, Harrison Edwards, Amos Storkey, and Oleg Klimov.
\newblock Exploration by random network distillation.
\newblock In \emph{International Conference on Learning Representations}, 2019.
\newblock URL \url{https://openreview.net/forum?id=H1lJJnR5Ym}.

\bibitem[Chen \& Peng(2019)Chen and Peng]{chen2019off}
Gang Chen and Yiming Peng.
\newblock Off-policy actor-critic in an ensemble: Achieving maximum general
  entropy and effective environment exploration in deep reinforcement learning.
\newblock \emph{arXiv preprint arXiv:1902.05551}, 2019.

\bibitem[Du et~al.(2019)Du, Han, Fang, Liu, Dai, and Tao]{du2019liir}
Yali Du, Lei Han, Meng Fang, Ji~Liu, Tianhong Dai, and Dacheng Tao.
\newblock Liir: Learning individual intrinsic reward in multi-agent
  reinforcement learning.
\newblock \emph{Advances in Neural Information Processing Systems},
  32:\penalty0 4403--4414, 2019.

\bibitem[Fayad \& Ibrahim(2021)Fayad and Ibrahim]{fayad2021behavior}
Ammar Fayad and Majd Ibrahim.
\newblock Behavior-guided actor-critic: Improving exploration via learning
  policy behavior representation for deep reinforcement learning.
\newblock \emph{arXiv preprint arXiv:2104.04424}, 2021.

\bibitem[Foerster et~al.(2018)Foerster, Farquhar, Afouras, Nardelli, and
  Whiteson]{foerster2018counterfactual}
Jakob Foerster, Gregory Farquhar, Triantafyllos Afouras, Nantas Nardelli, and
  Shimon Whiteson.
\newblock Counterfactual multi-agent policy gradients.
\newblock In \emph{Proceedings of the AAAI Conference on Artificial
  Intelligence}, volume~32, 2018.

\bibitem[Fujimoto et~al.(2018)Fujimoto, Van~Hoof, and
  Meger]{fujimoto2018addressing}
Scott Fujimoto, Herke Van~Hoof, and David Meger.
\newblock Addressing function approximation error in actor-critic methods.
\newblock \emph{arXiv preprint arXiv:1802.09477}, 2018.

\bibitem[Glorot \& Bengio(2010)Glorot and Bengio]{glorot2010understanding}
Xavier Glorot and Yoshua Bengio.
\newblock Understanding the difficulty of training deep feedforward neural
  networks.
\newblock In \emph{Proceedings of the thirteenth international conference on
  artificial intelligence and statistics}, pp.\  249--256. JMLR Workshop and
  Conference Proceedings, 2010.

\bibitem[Haarnoja et~al.(2018)Haarnoja, Zhou, Abbeel, and
  Levine]{haarnoja2018soft}
Tuomas Haarnoja, Aurick Zhou, Pieter Abbeel, and Sergey Levine.
\newblock Soft actor-critic: Off-policy maximum entropy deep reinforcement
  learning with a stochastic actor.
\newblock In \emph{International conference on machine learning}, pp.\
  1861--1870. PMLR, 2018.

\bibitem[Han et~al.(2020)Han, Zhang, Wang, and Mao]{han2020curiosity}
Gao-Jie Han, Xiao-Fang Zhang, Hao Wang, and Chen-Guang Mao.
\newblock Curiosity-driven variational autoencoder for deep q network.
\newblock In \emph{Pacific-Asia Conference on Knowledge Discovery and Data
  Mining}, pp.\  764--775. Springer, 2020.

\bibitem[Hausknecht \& Stone(2015)Hausknecht and Stone]{hausknecht2015deep}
Matthew Hausknecht and Peter Stone.
\newblock Deep recurrent q-learning for partially observable mdps.
\newblock In \emph{2015 aaai fall symposium series}, 2015.

\bibitem[He et~al.(2015)He, Zhang, Ren, and Sun]{he2015delving}
Kaiming He, Xiangyu Zhang, Shaoqing Ren, and Jian Sun.
\newblock Delving deep into rectifiers: Surpassing human-level performance on
  imagenet classification.
\newblock In \emph{Proceedings of the IEEE international conference on computer
  vision}, pp.\  1026--1034, 2015.

\bibitem[Hostallero et~al.(2020)Hostallero, Kim, Moon, Son, Kang, and
  Yi]{hostallero2020inducing}
David~Earl Hostallero, Daewoo Kim, Sangwoo Moon, Kyunghwan Son, Wan~Ju Kang,
  and Yung Yi.
\newblock Inducing cooperation through reward reshaping based on peer
  evaluations in deep multi-agent reinforcement learning.
\newblock In \emph{Proceedings of the 19th International Conference on
  Autonomous Agents and MultiAgent Systems}, pp.\  520--528, 2020.

\bibitem[Iqbal \& Sha(2019)Iqbal and Sha]{iqbal2019coordinated}
Shariq Iqbal and Fei Sha.
\newblock Coordinated exploration via intrinsic rewards for multi-agent
  reinforcement learning.
\newblock \emph{arXiv preprint arXiv:1905.12127}, 2019.

\bibitem[Jaques et~al.(2019)Jaques, Lazaridou, Hughes, Gulcehre, Ortega,
  Strouse, Leibo, and De~Freitas]{jaques2019social}
Natasha Jaques, Angeliki Lazaridou, Edward Hughes, Caglar Gulcehre, Pedro
  Ortega, DJ~Strouse, Joel~Z Leibo, and Nando De~Freitas.
\newblock Social influence as intrinsic motivation for multi-agent deep
  reinforcement learning.
\newblock In \emph{International Conference on Machine Learning}, pp.\
  3040--3049. PMLR, 2019.

\bibitem[Kingma \& Ba(2014)Kingma and Ba]{kingma2014adam}
Diederik~P Kingma and Jimmy Ba.
\newblock Adam: A method for stochastic optimization.
\newblock \emph{arXiv preprint arXiv:1412.6980}, 2014.

\bibitem[Klissarov et~al.(2019)Klissarov, Islam, Khetarpal, and
  Precup]{klissarov2019variational}
Martin Klissarov, Riashat Islam, Khimya Khetarpal, and Doina Precup.
\newblock Variational state encoding as intrinsic motivation in reinforcement
  learning.
\newblock In \emph{Task-Agnostic Reinforcement Learning Workshop at Proceedings
  of the International Conference on Learning Representations}, 2019.

\bibitem[Lillicrap et~al.(2015)Lillicrap, Hunt, Pritzel, Heess, Erez, Tassa,
  Silver, and Wierstra]{lillicrap2015continuous}
Timothy~P Lillicrap, Jonathan~J Hunt, Alexander Pritzel, Nicolas Heess, Tom
  Erez, Yuval Tassa, David Silver, and Daan Wierstra.
\newblock Continuous control with deep reinforcement learning.
\newblock \emph{arXiv preprint arXiv:1509.02971}, 2015.

\bibitem[Littman(1994)]{littman1994markov}
Michael~L Littman.
\newblock Markov games as a framework for multi-agent reinforcement learning.
\newblock In \emph{Machine learning proceedings 1994}, pp.\  157--163.
  Elsevier, 1994.

\bibitem[Liu et~al.(2021)Liu, Jain, Yeh, and Schwing]{liu2021cooperative}
Iou-Jen Liu, Unnat Jain, Raymond~A Yeh, and Alexander Schwing.
\newblock Cooperative exploration for multi-agent deep reinforcement learning.
\newblock In \emph{International Conference on Machine Learning}, pp.\
  6826--6836. PMLR, 2021.

\bibitem[Lowe et~al.(2017)Lowe, Wu, Tamar, Harb, Abbeel, and
  Mordatch]{lowe2017multi}
Ryan Lowe, Yi~Wu, Aviv Tamar, Jean Harb, Pieter Abbeel, and Igor Mordatch.
\newblock Multi-agent actor-critic for mixed cooperative-competitive
  environments.
\newblock \emph{arXiv preprint arXiv:1706.02275}, 2017.

\bibitem[Martin et~al.(2017)Martin, Sasikumar, Everitt, and
  Hutter]{martin2017count}
Jarryd Martin, Suraj~Narayanan Sasikumar, Tom Everitt, and Marcus Hutter.
\newblock Count-based exploration in feature space for reinforcement learning.
\newblock \emph{arXiv preprint arXiv:1706.08090}, 2017.

\bibitem[Matignon et~al.(2012)Matignon, Laurent, and
  Le~Fort-Piat]{matignon2012independent}
Laetitia Matignon, Guillaume~J Laurent, and Nadine Le~Fort-Piat.
\newblock Independent reinforcement learners in cooperative markov games: a
  survey regarding coordination problems.
\newblock \emph{Knowledge Engineering Review}, 27\penalty0 (1):\penalty0 1--31,
  2012.

\bibitem[Mordatch \& Abbeel(2018)Mordatch and Abbeel]{mordatch2018emergence}
Igor Mordatch and Pieter Abbeel.
\newblock Emergence of grounded compositional language in multi-agent
  populations.
\newblock In \emph{Proceedings of the AAAI Conference on Artificial
  Intelligence}, volume~32, 2018.

\bibitem[Rashid et~al.(2018)Rashid, Samvelyan, Schroeder, Farquhar, Foerster,
  and Whiteson]{rashid2018qmix}
Tabish Rashid, Mikayel Samvelyan, Christian Schroeder, Gregory Farquhar, Jakob
  Foerster, and Shimon Whiteson.
\newblock Qmix: Monotonic value function factorisation for deep multi-agent
  reinforcement learning.
\newblock In \emph{International Conference on Machine Learning}, pp.\
  4295--4304. PMLR, 2018.

\bibitem[Roy et~al.(2019)Roy, Barde, Harvey, Nowrouzezahrai, and
  Pal]{roy2019promoting}
Julien Roy, Paul Barde, F{\'e}lix~G Harvey, Derek Nowrouzezahrai, and
  Christopher Pal.
\newblock Promoting coordination through policy regularization in multi-agent
  deep reinforcement learning.
\newblock \emph{arXiv preprint arXiv:1908.02269}, 2019.

\bibitem[Samvelyan et~al.(2019)Samvelyan, Rashid, De~Witt, Farquhar, Nardelli,
  Rudner, Hung, Torr, Foerster, and Whiteson]{samvelyan2019starcraft}
Mikayel Samvelyan, Tabish Rashid, Christian~Schroeder De~Witt, Gregory
  Farquhar, Nantas Nardelli, Tim~GJ Rudner, Chia-Man Hung, Philip~HS Torr,
  Jakob Foerster, and Shimon Whiteson.
\newblock The starcraft multi-agent challenge.
\newblock \emph{arXiv preprint arXiv:1902.04043}, 2019.

\bibitem[Schulman et~al.(2015)Schulman, Levine, Abbeel, Jordan, and
  Moritz]{Schulman2015TrustRP}
John Schulman, Sergey Levine, Pieter Abbeel, Michael~I. Jordan, and Philipp
  Moritz.
\newblock Trust region policy optimization.
\newblock In \emph{32nd International Conference on Machine Learning}, 2015.

\bibitem[Schulman et~al.(2017)Schulman, Wolski, Dhariwal, Radford, and
  Klimov]{schulman2017proximal}
John Schulman, Filip Wolski, Prafulla Dhariwal, Alec Radford, and Oleg Klimov.
\newblock Proximal policy optimization algorithms.
\newblock \emph{arXiv preprint arXiv:1707.06347}, 2017.

\bibitem[Sunehag et~al.(2018)Sunehag, Lever, Gruslys, Czarnecki, Zambaldi,
  Jaderberg, Lanctot, Sonnerat, Leibo, Tuyls, and
  Graepel]{Sunehag2018ValueDecompositionNF}
Peter Sunehag, Guy Lever, A.~Gruslys, Wojciech Czarnecki, V.~Zambaldi, Max
  Jaderberg, Marc Lanctot, Nicolas Sonnerat, Joel~Z. Leibo, K.~Tuyls, and
  T.~Graepel.
\newblock Value-decomposition networks for cooperative multi-agent learning.
\newblock \emph{ArXiv}, abs/1706.05296, 2018.

\bibitem[Sutton et~al.(1998)Sutton, Barto, et~al.]{sutton1998introduction}
Richard~S Sutton, Andrew~G Barto, et~al.
\newblock \emph{Introduction to reinforcement learning}, volume 135.
\newblock MIT press Cambridge, 1998.

\bibitem[Tan et~al.(2018)Tan, Zhang, Coumans, Iscen, Bai, Hafner, Bohez, and
  Vanhoucke]{tan2018sim}
Jie Tan, Tingnan Zhang, Erwin Coumans, Atil Iscen, Yunfei Bai, Danijar Hafner,
  Steven Bohez, and Vincent Vanhoucke.
\newblock Sim-to-real: Learning agile locomotion for quadruped robots.
\newblock \emph{arXiv preprint arXiv:1804.10332}, 2018.

\bibitem[Tan(1993)]{Tan93multi-agentreinforcement}
Ming Tan.
\newblock Multi-agent reinforcement learning: Independent vs. cooperative
  agents.
\newblock In \emph{In Proceedings of the Tenth International Conference on
  Machine Learning}, pp.\  330--337. Morgan Kaufmann, 1993.

\bibitem[Tang et~al.(2017)Tang, Houthooft, Foote, Stooke, Chen, Duan, Schulman,
  Turck, and Abbeel]{Tang2017ExplorationAS}
Haoran Tang, Rein Houthooft, Davis Foote, Adam Stooke, Xi~Chen, Yan Duan, John
  Schulman, Filip~De Turck, and Pieter Abbeel.
\newblock \#exploration: A study of count-based exploration for deep
  reinforcement learning.
\newblock In \emph{31st Conference on Neural Information Processing Systems},
  2017.

\bibitem[Todorov et~al.(2012)Todorov, Erez, and Tassa]{todorov2012mujoco}
Emanuel Todorov, Tom Erez, and Yuval Tassa.
\newblock Mujoco: A physics engine for model-based control.
\newblock In \emph{2012 IEEE/RSJ International Conference on Intelligent Robots
  and Systems}, pp.\  5026--5033. IEEE, 2012.

\bibitem[Veli{\v{c}}kovi{\'c} et~al.(2017)Veli{\v{c}}kovi{\'c}, Cucurull,
  Casanova, Romero, Lio, and Bengio]{velivckovic2017graph}
Petar Veli{\v{c}}kovi{\'c}, Guillem Cucurull, Arantxa Casanova, Adriana Romero,
  Pietro Lio, and Yoshua Bengio.
\newblock Graph attention networks.
\newblock \emph{arXiv preprint arXiv:1710.10903}, 2017.

\bibitem[Wang et~al.(2019)Wang, Wang, Wu, and Zhang]{wang2019influence}
Tonghan Wang, Jianhao Wang, Yi~Wu, and Chongjie Zhang.
\newblock Influence-based multi-agent exploration.
\newblock \emph{arXiv preprint arXiv:1910.05512}, 2019.

\bibitem[Wu et~al.(2017)Wu, Mansimov, Grosse, Liao, and Ba]{Wu2017ScalableTM}
Yuhuai Wu, Elman Mansimov, Roger~B. Grosse, Shun Liao, and Jimmy Ba.
\newblock Scalable trust-region method for deep reinforcement learning using
  kronecker-factored approximation.
\newblock \emph{ArXiv}, abs/1708.05144, 2017.

\bibitem[Yang et~al.(2020)Yang, Li, Farajtabar, Sunehag, Hughes, and
  Zha]{yang2020learning}
Jiachen Yang, Ang Li, Mehrdad Farajtabar, Peter Sunehag, Edward Hughes, and
  Hongyuan Zha.
\newblock Learning to incentivize other learning agents.
\newblock \emph{arXiv preprint arXiv:2006.06051}, 2020.

\bibitem[Yang et~al.(2018)Yang, Luo, Li, Zhou, Zhang, and Wang]{yang2018mean}
Yaodong Yang, Rui Luo, Minne Li, Ming Zhou, Weinan Zhang, and Jun Wang.
\newblock Mean field multi-agent reinforcement learning.
\newblock In \emph{International Conference on Machine Learning}, pp.\
  5571--5580. PMLR, 2018.

\bibitem[Zhang et~al.(2019)Zhang, Yu, and Turk]{zhang2019learning}
Yunbo Zhang, Wenhao Yu, and Greg Turk.
\newblock Learning novel policies for tasks.
\newblock In \emph{International Conference on Machine Learning}, pp.\
  7483--7492. PMLR, 2019.

\bibitem[Zhou et~al.(2020{\natexlab{a}})Zhou, Cui, Zhang, Yang, Liu, and
  Sun]{Zhou2020GraphNN}
Jie Zhou, Ganqu Cui, Zhengyan Zhang, Cheng Yang, Zhiyuan Liu, and Maosong Sun.
\newblock Graph neural networks: A review of methods and applications.
\newblock \emph{AI Open}, 1:\penalty0 57--81, 2020{\natexlab{a}}.

\bibitem[Zhou et~al.(2020{\natexlab{b}})Zhou, Liu, Sui, Li, and
  Chung]{zhou2020learning}
Meng Zhou, Ziyu Liu, Pengwei Sui, Yixuan Li, and Yuk~Ying Chung.
\newblock Learning implicit credit assignment for cooperative multi-agent
  reinforcement learning.
\newblock \emph{arXiv preprint arXiv:2007.02529}, 2020{\natexlab{b}}.

\end{thebibliography}
\bibliographystyle{iclr2022_conference}
\newpage
\appendix
\section{Proofs}\label{APP}
\textbf{Proposition 1.} Consider an $L$-action setting of $n$ agents. In expectation, the number of steps $T$ needed to visit all $L^n$ action configurations at least once without coordinated exploration grows at least exponentially with the number of agents. More concretely,  $\mathbb{E}[T]=\Omega(nL^n )$.
\begin{proof}
Let $M=L^n$. Since agents tend to visit different action configurations with no coordinated behavior, one can equivalently say that agents uniformly pick a  configuration out of all $L^n$ possible configurations at each step. Let $T_k$ be the number of steps to visit the $k$-th distinct configuration after covering $k-1$ distinct action tuples. Observe that:\begin{equation}\label{ET}
    \mathbb{E}[T]=\sum_{k=1}^{M}\mathbb{E}[T_k]
\end{equation} 
Now, $\text{Pr}[T_k=i]= \bigl(\frac{k-1}{M}\bigl)^{i-1}\bigl(1-\frac{k-1}{M}\bigl)$ meaning that $T_k$ follows a geometric distribution. Thus, \begin{equation}
    \mathbb{E}[T_k]=\sum_{i=1}^{\infty} \biggl(\frac{k-1}{M}\biggl)^{i-1}\biggl(\frac{M-k+1}{M}\biggl) i= \frac{M}{M-k+1}
\end{equation}
Getting back to Eq. (\ref{ET}),\begin{equation}
    \mathbb{E}[T]=M\sum_{k=1}^{M}\frac{1}{M-k+1}=M\sum_{k=1}^{M}\frac{1}{k}>M\int_{1}^{M} \frac{1}{x} dx = M\ln M = n L^n \ln L
\end{equation}
And the conclusion follows.
\end{proof}

\section{Additional Experiments on Continuous Environments}\label{ADD}

Since the formulation of $F$ needs a shared buffer, SAC and TD3 stand as the best off-policy candidates to be incorporated with our framework, as they have shown great performances on many benchmarks. SAC, however, uses stochastic policies in general which makes it infeasible to combine with the formulation of $r^{\text{int}}$. Therefore, we use TD3 as our learning model to measure its performance on a suite of PyBullet \citep{tan2018sim} continuous control tasks, interfaced through OpenAI Gym \citep{brockman2016openai}. While many previous works utilized the Mujoco \citep{todorov2012mujoco} physics engine to simulate the system dynamics of these tasks, we found it better to evaluate our method on benchmark problems powered by PyBullet simulator since it is widely reported that PyBullet problems are harder to solve \citep{tan2018sim} when compared to Mujoco. Also, Pybullet is license-free, unlike Mujoco that is only available to its license holders.
\begin{figure}[h!]
 
\centering
\begin{subfigure}[b]{0.25\linewidth}
    \includegraphics[width=\linewidth]{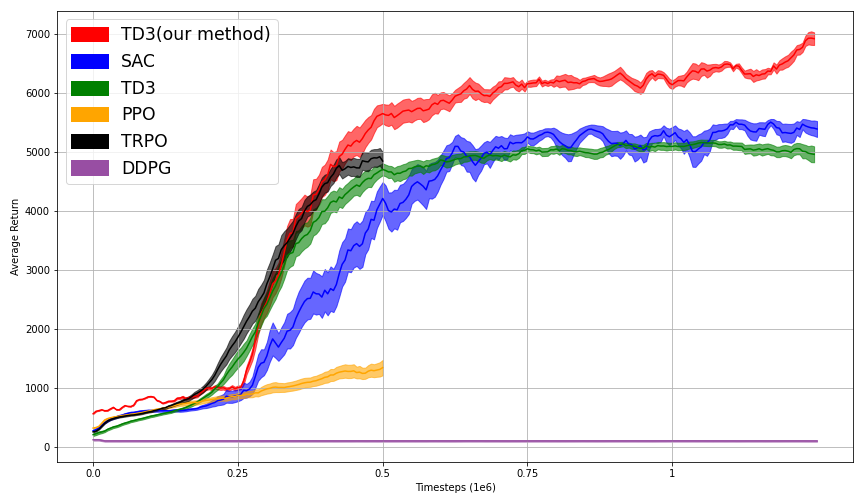}
     \caption{Humanoid-v3}
\end{subfigure}
  \begin{subfigure}[b]{0.25\linewidth}
    \includegraphics[width=\linewidth]{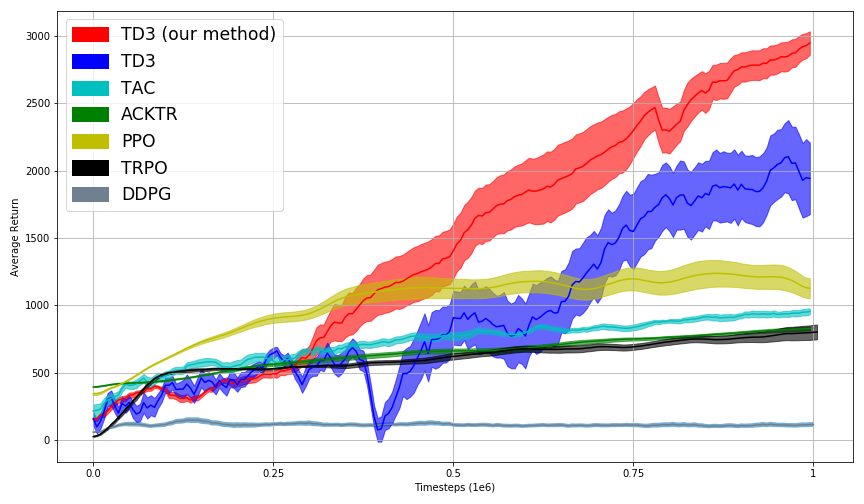}
     \caption{Ant}
  \end{subfigure}
\begin{subfigure}[b]{0.25\linewidth}
    \includegraphics[width=\linewidth]{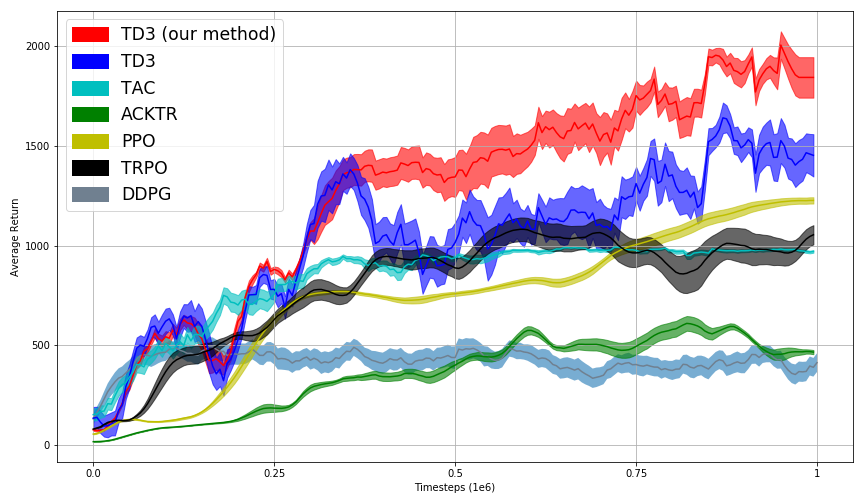}
     \caption{Walker2D}
\end{subfigure}
\begin{subfigure}[b]{0.25\linewidth}
    \includegraphics[width=\linewidth]{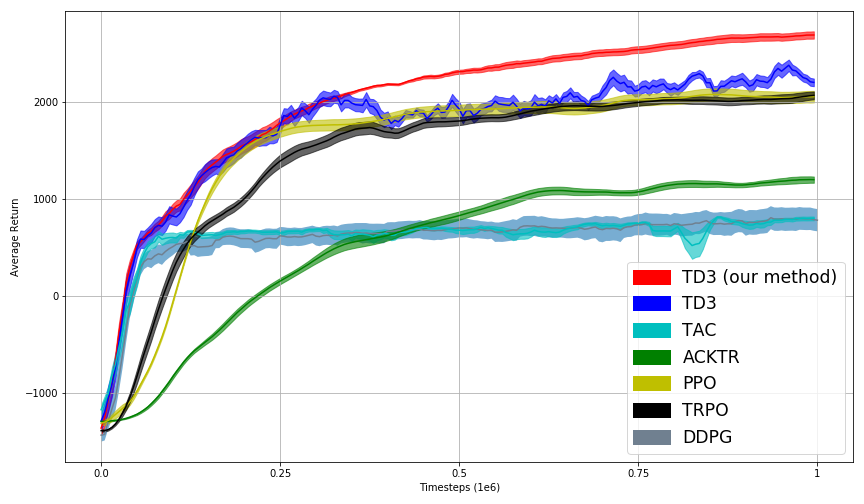}
     \caption{HalfCheetah}
  \end{subfigure} 
\begin{subfigure}[b]{0.25\linewidth}
    \includegraphics[width=\linewidth]{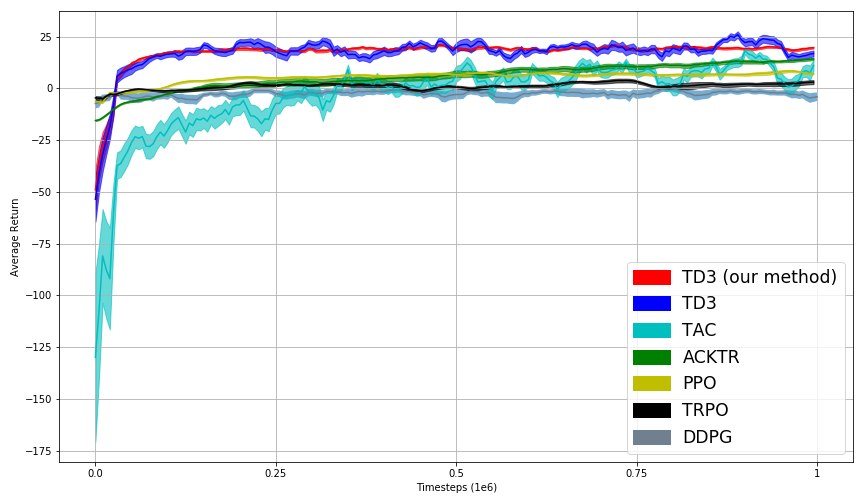}
     \caption{Reacher}
\end{subfigure}

\caption{Learning curves for the OpenAI gym continuous control tasks. The shaded region represents quarter a standard deviation of the average evaluation. Curves are smoothed for visual clarity.}
\label{1}
\end{figure}\\

We compare our method to the original twin delayed deep deterministic policy gradients (TD3) \citep{fujimoto2018addressing}; soft actor critic (SAC) \citep{haarnoja2018soft}; proximal policy optimization (PPO) \citep{schulman2017proximal}, a stable and efficient on-policy policy gradient algorithm; deep deterministic policy gradient (DDPG); trust region policy optimization (TRPO) \citep{Schulman2015TrustRP}; Tsallis actor-critic (TAC) \citep{chen2019off}, a recent off-policy algorithm for learning maximum entropy policies, where we use the implementation of the authors\footnote{\url{https://github.com/haarnoja/sac}} \footnote{\url{ https://github.com/yimingpeng/sac-master}}; and Actor-Critic using Kronecker-Factored Trust Region (ACKTR) \citep{Wu2017ScalableTM}, as implemented by OpenAI’s baselines repository \footnote{\url{https://github.com/openai/baselines}}. Each task is run for at least 1 million time steps and the average return of 15 episodes is reported every 5000 time steps. To enable reproducibility, each experiment is conducted on 10 random seeds of Gym simulator and network initialization. Results of the best performing agent of the two across different methods are reported in Figure (\ref{1}).

\section{Training details}\label{app:training_details}
\subsection{General Configurations}
 We use a buffer-size of $10^6$ entries and a batch-size of $1024$. We collect $100$ transitions by interacting with the environment for each learning update. For all tasks in our hyper-parameter searches, we train the agents for $15,000$ episodes of $100$ steps and then re-train the best configuration for each algorithm-environment pair for twice as long ($30,000$ episodes) to ensure full convergence for the final evaluation. We use a discount factor $\gamma$ of $0.95$, an influence importance temperature $\beta$ of $0.1$, and a gradient clipping threshold of $0.5$ in all experiments unless otherwise specified. Each cloned critic is updated 4 time per step.

\subsection{Sparse Push Box and Sparse Secret Room, MAPE, \& Gym}
We use the Adam optimizer \citep{kingma2014adam} to perform parameter updates. All models (actors, critics and proxy critics) are parametrized by feedforward networks containing two hidden layers of $128$ units excpet for the autoencoder network where we use 7 hidden layers with dimensions (128, 64, 12, 3, 12, 64, 128), respectively. All models' parameters are initialized using Glorot Initialization method \citep{glorot2010understanding}; while the autoencoder's parameters are initialized using Kaiming method \citep{he2015delving}. We employ the Rectified Linear Unit (ReLU)  as activation function and layer normalization \citep{ba2016layer} on the pre-activations unit to stabilize the learning.

\begin{table}[H]
\centering
\begin{sc}
\caption{Best found hyper-parameters for the Sparse-reward tasks, MAPE, \& Gym environments}
\resizebox{\textwidth}{!}{%
\begin{tabular}{lccccc}
\label{table:hp_chase}
Hyper-parameter          &Push Box      &Secret Room             &MAPE  &Humanoid-v3        &Gym (except for Humanoid-v3)   
\\ \hline

$\lambda_\pi$        &0.10             &0.10             &0.01               &0.10           &$0.10$\\
$\{\omega_i\}_1^{n-1}$        &0.10             &0.10              &0.01               &0.10        &$0.10$\\
$\beta$        &0.15             &0.10              &0.10               &0.15         &$0.1$\\
\end{tabular}}
\end{sc}
\end{table}

\subsection{SMAC}
The architecture of all agent networks is a DRQN \citep{hausknecht2015deep} with a recurrent layer comprised of a GRU with a $64$-dimensional hidden state, with a fully-connected layer before and after. All neural networks are trained using RMSprop ($\alpha=0.99$ with no weight decay or momentum) with learning rate $5\times 10^{-4}$.

\begin{table}[H]
\centering
\begin{small}
\caption{Best found hyper-parameters for the SMAC environments}
\resizebox{0.7\textwidth}{!}{%
\begin{tabular}{lccccc}
\label{table:hp_chase}
Hyper-parameter          &Corridor      &5m\_vs\_6m            &3s\_vs\_5z  &2s3z           
\\ \hline

$\lambda_\pi$        &0.09             &0.03              &0.01              &$0.01$\\
$\{\omega_i\}_1^{n-1}$         &0.03             &0.03              &0.01              &$0.01$         \\
$\beta$        &0.15             &0.15              &0.10               &0.10         \\

\end{tabular}}
\end{small}
\end{table}

\end{document}